\documentclass[runningheads]{llncs}
\usepackage[T1]{fontenc}
\usepackage{graphicx}
\usepackage{amsfonts}
\usepackage{amsmath}
\usepackage{booktabs} 
\usepackage{multirow}
\usepackage{color}
\begin{document}
\title{S‑GBT: Smooth Growth Bound Tensor for Certified Robustness Against Word Substitution Attacks in NLP}
\titlerunning{S‑GBT: Smooth Growth Bound Tensor}
%
\author{Mohammed Bouri\inst{1,3} \and
Mohammed Erradi\inst{1,2} \and
Adnane Saoud\inst{1}}
\authorrunning{M. Bouri et al.}
%
\institute{College of Computing, Mohammed VI Polytechnic University, Morocco \and
ENSIAS, University Mohamed V of Rabat, Morocco \and CID Development, Morocco\\
\email{\{mohammed.bouri,mohammed.erradi,adnane.saoud\}@um6p.ma}}
\maketitle              
\begin{abstract}
Despite recent progress in Natural Language Processing (NLP), models remain vulnerable to word substitution attacks. Most existing defenses focus on first order sensitivity and measure how much the output changes when the input is slightly perturbed. However, they ignore how this sensitivity evolves, which is described by curvature. When gradients vary sharply, models can still fail. This paper introduces the Smooth Growth Bound Tensor (S-GBT), a second order method that bounds the Hessian element-wise, for which we provide formal theoretical proofs on the resulting robustness bounds. A regularization term is added during training to minimize these bounds. This yields tighter certified robustness against word substitution attacks. The change in the output under word substitution is bounded by both a linear term and a quadratic term. S-GBT is derived for two architectures: Long Short-Term Memory (LSTM) and Convolutional Neural Networks (CNN). The method is integrated directly into the training objective. Its effectiveness is evaluated on multiple benchmark datasets. The results show that combining first and second order regularization improves certified robust accuracy by up to 23.4\% compared to prior methods, while clean accuracy remains competitive. These findings indicate that controlling both the gradient and its variation is a promising direction for building more robust models.

\end{abstract}

\section{Introduction}
Recent progress in deep learning has significantly improved the performance of many Natural Language Processing (NLP) systems. In particular, advances in language modeling have enabled strong results in tasks such as sentiment analysis, text classification, and question answering. Despite these successes, research has revealed an important weakness of deep neural networks: their predictions can be easily altered by small modifications to the input text. These modified inputs, known as adversarial examples, are designed to cause a model to produce an incorrect output while keeping the text similar to the original.

The existence of adversarial examples was first highlighted in computer vision, where small perturbations in images were shown to mislead neural networks \cite{goodfellow2014explaining,qi2024exploring}. Since then, the study of adversarial robustness has extended to NLP. In text applications, adversarial examples raise serious concerns because they may affect the reliability of systems used in practical scenarios, including text classification tasks \cite{song2021universal}. Furthermore, although Large Language Models are widely used for generating text, recent work indicates that they can also be affected by adversarial perturbations when applied to classification tasks \cite{shayegani2023survey,wang2024generating,yang2024assessing}.

Adversarial attacks in NLP typically modify the input text in different ways. These attacks are commonly grouped into three categories: character level, sentence level, and word level attacks. Character level attacks introduce small changes to individual characters in a word \cite{javid18hotflip,eger2020hero}. In practice, such modifications can sometimes be detected or corrected using spelling correction techniques \cite{danish19combating}. Sentence level attacks generate paraphrased sentences \cite{pei2022generating,wang2019t3}, but they often alter the meaning of the original text, which limits their effectiveness in some settings. Word level attacks have received particular attention in recent studies. In this type of attack, selected words in a sentence are replaced with alternative words that have similar meanings \cite{alzantot-etal-2018-generating,rishabh21generating,shuhuai19generating,wang21adversarial,yuan20pso}. These substitutions usually preserve the grammar and overall meaning of the sentence, which makes the changes difficult for humans to notice. However, even a small number of substitutions may cause a neural model to change its prediction. Because of this property, word substitution attacks often achieve high success rates and remain difficult to defend against. For this reason, improving robustness against word-level perturbations is an important research problem.

Several defense methods have been proposed to address this challenge. A common strategy is Adversarial Training (AT), where adversarial examples are generated and added to the training data \cite{alzantot-etal-2018-generating,shuhuai19generating}. While this approach can increase robustness, it also introduces a high computational cost because generating adversarial examples requires complex search procedures in the discrete text space. To reduce this cost, later work proposed faster white-box attack methods that can be used during training \cite{dong21towards,wang21adversarial}. Nevertheless, AT remains significantly slower than standard training and does not provide formal robustness guarantees. Another direction focuses on certified robustness methods that aim to provide theoretical guarantees about the behavior of the model under perturbations. Interval Bound Propagation (IBP) has been applied to NLP models to compute lower bounds on model robustness \cite{huang19achieving,jia19certified}. Although this approach provides formal guarantees, it is often computationally demanding and tends to produce conservative robustness estimates. To address these limitations, Zhang et al. \cite{zhang-etal-2021-certified} proposed Abstractive Recursive Certification (ARC), which constructs a perturbation space based on predefined string transformations. This perturbation space is represented as a hyperrectangle and propagated through the model using IBP. However, ARC becomes less effective when the number of allowed word substitutions increases. This limitation highlights the need for more efficient and scalable certified robustness methods. Growth Bound Matrices (GBM) \cite{bouri2025bridging} were recently proposed to improve the robustness of NLP models by bounding the first derivatives of the model output with respect to the input. By minimizing these bounds during training, the model becomes less sensitive to small input perturbations. However, GBM controls only the gradient of the model, while the curvature, captured by the second derivatives, also affects how the output evolves under larger perturbations.

In this context, the paper introduces the Smooth Growth Bound Tensor (S-GBT). This method bounds the second partial derivatives of the model output with respect to the input. By penalizing both GBM and S-GBT during training, the model is encouraged to have small gradients and low curvature, leading to smoother decision boundaries and more stable predictions under input perturbations. The paper provides a theoretical framework that establishes certified robustness guarantees by controlling the variation of the model gradients under input perturbations. The proposed method ensures that output variations remain within controlled limits for all admissible perturbations. In addition, we derive a general formulation of S-GBT that applies to different neural architectures, including Long Short-Term Memory (LSTM) networks \cite{hochreiter1997long} and Convolutional Neural Networks (CNN) \cite{Kim14f}. Experiments on several benchmark datasets show that S-GBT significantly improves certified robustness. For example, on the Yahoo dataset, our method achieves $90.7\%$ certified robust accuracy, outperforming GBM by about $23.4\%$.

Our main contributions are summarized as follows:
\begin{itemize}
    \item \textbf{Certified robustness with S-GBT:} We introduce Smooth Growth Bound Tensor (S-GBT), a method that provides certified robustness by controlling the second-order variation of the model output, leading to stronger guarantees compared with first-order GBM.
    \item \textbf{Theoretical analysis of robustness:} We develop a theoretical analysis that provides certified robustness guarantees. The approach ensures that output variations remain within controlled limits under input perturbations in complex and large input spaces.
    \item \textbf{Empirical validation across models and datasets:} Extensive experiments on several NLP models and benchmark datasets show that minimizing S-GBT significantly improves adversarial robustness. Our method also outperforms existing defense baselines.
\end{itemize}
\section{Related Work}
Spelling and grammar correction methods have shown robustness against character-level and sentence-level attacks, which often violate grammatical rules \cite{ge-etal-2019-automatic,pruthi-etal-2019-combating}. However, these approaches are generally ineffective against word-level adversarial attacks.

Defense methods for word-level attacks mainly fall into several categories. Adversarial training (AT) is among the most widely used strategies \cite{alzantot-etal-2018-generating,goodfellow2014explaining,ivgi21achieving,li2020tavat,madry18towards,shuhuai19generating,wang2020infobert,zhou2020defense,zhu2019freelb}. It improves robustness by generating adversarial examples and incorporating them into the training process. For instance, Adversarial Training with FGPM enhanced by Logit Pairing (ATFL) \cite{wang21adversarial} creates adversarial texts using the Fast Gradient Projection Method and injects them into the training set. Region-based adversarial training \cite{dong21towards,zhou2020defense} further enhances robustness by optimizing the model within the convex region defined by a word embedding and its synonyms.

Another line of work focuses on certified defenses, which aim to guarantee robustness against a predefined set of perturbations \cite{huang19achieving,wang21certified,zeng21certified}. Interval Bound Propagation (IBP) methods \cite{jia19certified} compute upper and lower bounds on the model output under interval-constrained inputs to minimize the worst-case loss caused by word substitutions. However, IBP-based approaches often suffer from high computational cost and limited scalability. Randomized smoothing methods \cite{zhang2024text,yeetal2020safer,zeng2021certified} offer an alternative by constructing stochastic ensembles of perturbed inputs and using their statistical properties to certify robustness. Growth Bound Matrices (GBM) \cite{bouri2025bridging} provide another certified defense by bounding the first derivatives of the model output with respect to its input. Minimizing these bounds during training reduces the sensitivity of the model to small input perturbations. However, GBM only controls the gradient of the model but fails to capture its variation, which reflects changes in the model’s local behavior. In addition, many certified defense methods assume access to the attacker’s synonym set, which is often unrealistic in practice.
\section{Preliminaries}
\label{subsec:Preliminaries}
This paper works on text classification tasks. Let \(f: \mathcal{X} \rightarrow \mathcal{Y}\) be a model that assigns a label \(y \in \mathcal{Y}\) to an input \(x \in \mathcal{X}\). The input sentence is represented as a sequence of \(N\) words \(x = \langle w_1, w_2, \ldots, w_N \rangle\), where each word \(w_i \in \mathcal{V}\) and \(\mathcal{V}\) denotes the vocabulary. The label space is defined as \(\mathcal{Y} = \{y_1, y_2, \ldots, y_c\}\). Standard NLP models map the input \(x\) to a sequence of embedding vectors \(v_x = \langle v_{w_1}, v_{w_2}, \ldots, v_{w_N} \rangle\), and the prediction is obtained through a deep neural network that estimates \(p(y|v_x)\).

Robustness against synonym substitution attacks has been widely studied in~\cite{alzantot-etal-2018-generating,dong21towards,jia19certified,wang21natural,zhang2024text}. For each word \(w_i\), a synonym set \(\mathcal{S}(w_i)\) is constructed by selecting the \(k\) nearest neighbors of \(w_i\) within an Euclidean distance \(d_e\) in the embedding space. To preserve semantic consistency, the adversarial input space is defined as:
\begin{equation*}
\mathcal{S}_{adv}(x) = \{\langle w_1', \ldots, w_N' \rangle \mid w_i' \in \mathcal{S}(w_i) \cup \{w_i\}\}.
\end{equation*}

The objective is to certify the robustness of the classifier by ensuring that:
\begin{equation*}
\forall x' \in \mathcal{S}_{adv}(x), \quad f(x') = f(x) = y.
\end{equation*}

In other words, the prediction should remain unchanged when words in the sentence are replaced with their synonyms.


\section{Smooth Growth Bound Tensor (S‑GBT)}
\label{subsec:s_gbt}

In this section, we describe our proposed approach. We begin by defining the Smooth Growth Bound Tensor (S-GBT) and explain how incorporating it into the training objective can enhance the robustness of the model. We then provide the derivation of S-GBT for two neural architectures: Long Short-Term Memory (LSTM) networks and Convolutional Neural Networks (CNN).


\subsection{Definition and Theoretical Foundation}
Consider a mapping, 
\begin{equation}\label{eq:orig}
\begin{array}{l}
\mathcal{F} :\hspace{2.5mm} \mathbf{X} \to \mathbf{Y} \\
\quad\quad\hspace{2mm} x \mapsto \mathcal{F}(x)
\end{array}
\end{equation}
where \(\mathbf{X} \subseteq \mathbb{R}^{n_x}\) and \(\mathbf{Y} \subseteq \mathbb{R}^{n_y}\).

\begin{definition}[Growth Bound Matrix] 
\label{def:gbm}
A matrix \( \mathcal{M} \in \mathbb{R}^{n_y \times n_x} \) is said to be a \textit{Growth Bound Matrix (GBM)} for the mapping \( \mathcal{F} \) in (\ref{eq:orig}) if the following condition holds:
\begin{equation}\label{eq:GBM}
    \left\lVert \dfrac{\partial \mathcal{F}^i}{\partial x^j}(x)\right\rVert \leq (\mathcal{M})_{ij} \quad \forall (i,j) \in \mathcal{I}, \quad \forall x \in \mathbf{X}
\end{equation}
where \( \mathcal{I} = \{1,\ldots,n_y\} \times \{1,\ldots,n_x\} \) is a predefined index set, and \( (\mathcal{M})_{ij} \) denotes the entry in the \( i \)-th row and \( j \)-th column of \( \mathcal{M} \).
\end{definition}

\begin{definition}[Smooth Growth Bound Tensor]
\label{def:s_gbt}
A tensor $\mathcal{T}\in\mathbb{R}^{n_y\times n_x\times n_x}$ is called a {Smooth Growth Bound Tensor} for the mapping $\mathcal{F}$ in \eqref{eq:orig} if the following condition holds:\\
\begin{equation}
    \left\lVert\frac{\partial^2\mathcal{F}^i}{\partial x^j\partial x^k}(x)\right\rVert \le (\mathcal{T})_{i,j,k} \quad \forall i,j,k\in \mathcal{J} \quad \forall x \in \mathbf{X}    
\end{equation}
where \( \mathcal{J} = \{1,\dots,n_y\}\times \{1,\dots,n_x\}\times \{1,\dots,n_x\} \) is a predefined index set, and \((\mathcal{T})_{i,j,k}\) denotes the entry of the tensor corresponding to output component $i$ and the pair of input coordinates $(j,k)$.
\end{definition}

The tensor \(\mathcal{T}\) provides element-wise upper bounds on the second partial derivatives of the model output. Evaluating S-GBT therefore offers insight into the curvature of the model with respect to its input. By minimizing the entries of \(\mathcal{T}\), the curvature of the model is reduced, leading to smoother decision boundaries. This property improves robustness to input perturbations, especially when combined with the control of first-order sensitivity provided by the GBM framework.

\subsection{Robustness Analysis with S-GBT}
In this section, we present the theoretical analysis showing how the combination of S-GBT and the GBM provides certified robustness guarantees. A detailed proof of this result can be found in the Appendix \ref{sec:rob}.

\begin{proposition}
\label{prop:s_gbt_robustness}

Consider the mapping \( \mathcal{F} \) in (\ref{eq:orig}) and let a matrix \( \mathcal{M} \in \mathbb{R}^{n_y \times n_x} \) be its GBM, and $\mathcal{T}\in\mathbb{R}^{n_y\times n_x\times n_x}$ its S‑GBT. Given an input \( x \in \mathbf{X} \), a perturbation vector \( {\delta} \in \mathbb{R}^{n_x} \) and consider the perturbed input $x'\in \mathbf{X}$ such that $x'=x+{\delta}$. Then, for each component $\mathcal{F}_i$, $i\in \{1,\ldots,n_y\}$ , the following inequality holds:
\[
\begin{aligned}
\mathcal{F}^i(x)
- \sum_{j=1}^{n_x} (\mathcal{M})_{ij} \lvert \delta_j \rvert
-& \frac{1}{2} \sum_{j,k=1}^{n_x} (\mathcal{T})_{ijk} \lvert \delta_j \rvert \lvert \delta_k \rvert
\le \mathcal{F}^i(x') \\
&\le \mathcal{F}^i(x)
+ \sum_{j=1}^{n_x} (\mathcal{M})_{ij} \lvert \delta_j \rvert
+ \frac{1}{2} \sum_{j,k=1}^{n_x} (\mathcal{T})_{ijk} \lvert \delta_j \rvert \lvert \delta_k \rvert .
\end{aligned}
\]
where \({\delta_j}\) denotes the \(j\)-th component of the perturbation vector ${\delta}$.
\end{proposition}

This result indicates that the change in the \(i\)-th output component \(F^i\) under a perturbation \(\delta\) can be controlled by two contributions: a first-order term related to GBM and a second-order term associated with S-GBT. By reducing both \(\mathcal{M}\) and \(\mathcal{T}\), the model becomes more stable with respect to input changes and exhibits lower curvature. Consequently, the predictions remain consistent between the original and perturbed inputs, providing a certified robustness guarantee, even when the perturbation is large enough for second-order effects to play a role.

\subsection{Overall Training Objective}

The goal is to train a model that fits the data while limiting both its first-order sensitivity (through GBM) and its curvature (through S-GBT). To achieve this, we add two regularization terms to the training objective of the neural network, which encourages robustness to adversarial perturbations. The resulting objective function is defined as:
\begin{equation}
    \label{eq:objective_sgbt}
    \mathcal{L}(x,y)= \,\mathcal{L}_{\text{ce}}(f(x), y) + \beta \cdot \mathcal{L}_{\text{GBM}} + \gamma \cdot \mathcal{L}_{\text{SGBT}},
\end{equation}

where
\[
\mathcal{L}_{\text{GBM}} = \sum_{i=1}^{n_y}\sum_{j=1}^{n_x} (\mathcal{M})_{ij}, \qquad
\mathcal{L}_{\text{SGBT}} = \sum_{i=1}^{n_y}\sum_{j=1}^{n_x}\sum_{k=1}^{n_x} (\mathcal{T})_{i,j,k}.
\]

Here, \(\mathcal{L}_{\text{ce}}(\cdot,\cdot)\) denotes the cross-entropy loss used to optimize the classifier for prediction accuracy. The matrix \(\mathcal{M}\) corresponds to the first-order GBM defined in Eq.~\eqref{def:gbm}, while the tensor \(\mathcal{T}\) represents the second-order S-GBT introduced in Eq.~\eqref{def:s_gbt}. Their explicit forms for each architecture (LSTM and CNN) are provided in Subsection~\ref{subsec:models}. The hyperparameters \(\beta\) and \(\gamma\) control the balance between classification accuracy, sensitivity reduction, and curvature regularization. By incorporating both regularization terms into the objective, the trained model obtains smoother decision boundaries and improved robustness to adversarial perturbations.

\subsection{Architecture-Specific S-GBT Formulations}
\label{subsec:models}
This paper derives architecture-specific S-GBTs for two models: LSTM and CNN.

\textbf{Long Short-Term Memory (LSTM)} networks are an extension of Recurrent Neural Networks (RNNs) designed to mitigate the vanishing and exploding gradient issues encountered during the training of standard RNNs \cite{hochreiter1997long}. 

Consider a sentence composed of \(N\) word embeddings \(v_x = \langle v_{w_1}, v_{w_2}, \ldots, v_{w_N} \rangle\), where \(v_{w_t} \in \mathbb{R}^{d_0}\) denotes the embedding vector of the word at position \(t\). An LSTM cell maintains two internal states: the cell state \(c_t \in \mathbb{R}^d\) and the hidden state \(h_t \in \mathbb{R}^d\). These states are updated according to
\begin{align}
    c_t &= f_t \odot c_{t-1} + I_t \odot g_t \label{eq:ct} \\
    h_t &= o_t \odot \tanh(c_t) \label{eq:ht}
\end{align}
where \(\odot\) denotes element-wise multiplication. The variables \(I_t\), \(f_t\), \(g_t\), and \(o_t\) correspond to the input gate, forget gate, cell gate, and output gate, respectively. Each gate is associated with its own learnable parameters and regulates the flow of information through the network.

From Eqs.~(\ref{eq:ct}) and~(\ref{eq:ht}), the LSTM cell can be viewed as an input-output mapping \(\mathcal{F}\), as in Eq.\eqref{eq:orig}, formally defined as:
\begin{equation}
    \label{eq:lstm}
    \begin{aligned}
        \mathcal{F} :\qquad \mathbf{V} \times \mathbf{H} \times \mathbf{C} &\to \mathbb{R}^d \\
        (v_{w_t}, h_{t-1}, c_{t-1}) &\mapsto o_t \odot \tanh(f_t \odot c_{t-1} + I_t \odot g_t)
    \end{aligned}
\end{equation}

where the input \(x=(v_{w_t}, h_{t-1}, c_{t-1}) \in \mathbf{V} \times \mathbf{H} \times \mathbf{C} \subseteq \mathbb{R}^{d_0+2d}\) and the output \(y = h_t = \mathcal{F}(v_{w_t}, h_{t-1}, c_{t-1}) \in \mathbb{R}^d\). Here, \(\mathbf{V}\), \(\mathbf{H}\), and \(\mathbf{C}\) denote the domains of the word embeddings, hidden states, and cell states, respectively.

The following result provides an explicit expression for the S-GBT of the LSTM cell. A detailed proof of this result can be found in the Appendix \ref{sec:proofs}.
\begin{proposition}
\label{prop:lstm}
    Consider the map $\mathcal{F}$ describing the input-output model of an LSTM cell defined in Eq.\eqref{eq:lstm}. The S-GBT of the map $\mathcal{F}$ can be expressed as follows:

    For all \(i\in\{1,\dots,d\}\) and each pair of input indices \(j,k\in\{1,\dots,d_0+2d\}\)
    \[
(\mathcal{T})_{i,j,k} = \max\bigl( |(\underline{\mathcal{T}})_{i,j,k}|, |(\overline{\mathcal{T}})_{i,j,k}| \bigr).
\]
where, for all input \(x=(v_{w_t}, h_{t-1}, c_{t-1}) \in \mathbf{X}= (\mathbf{V} \times \mathbf{H} \times \mathbf{C})\)
    \begin{equation*}
    (\underline{\mathcal{T}})_{i,j,k} =
    \begin{cases}
          \min\{\hspace{1mm}({\mathcal{T}}_{vv})_{i,j,k} (x)
    \mid x \in \mathbf{X}\} \hspace{2mm} \text{if} \hspace{2mm} j,k\in \mathcal{J}_{v}\\
        \min\{\hspace{1mm}({\mathcal{T}}_{vh})_{i,j,k} (x)
    \mid x \in \mathbf{X}\} \hspace{2mm} \text{if} \hspace{2mm} (j,k)\in \mathcal{J}_{v}\times \mathcal{J}_{h}\\                  
        \min\{\hspace{1mm}({\mathcal{T}}_{vc})_{i,j,k} (x)
    \mid x \in \mathbf{X}\} \hspace{2mm} \text{if} \hspace{2mm} (j,k)\in \mathcal{J}_{v}\times \mathcal{J}_{c}\\
    \min\{\hspace{1mm}({\mathcal{T}}_{hh})_{i,j,k} (x)
    \mid x \in \mathbf{X}\} \hspace{2mm} \text{if} \hspace{2mm} j,k\in \mathcal{J}_{h}\\
        \min\{\hspace{1mm}({\mathcal{T}}_{hc})_{i,j,k} (x)
    \mid x \in \mathbf{X}\} \hspace{2mm} \text{if} \hspace{2mm} (j,k)\in \mathcal{J}_{h}\times \mathcal{J}_{c}\\                  
    \min\{\hspace{1mm}({\mathcal{T}}_{cc})_{i,j,k} (x)
    \mid x \in \mathbf{X}\} \hspace{2mm} \text{if} \hspace{2mm} j,k\in \mathcal{J}_{c}
    \end{cases}
    \end{equation*}

    \begin{equation*}
    (\underline{\mathcal{T}})_{i,j,k} =
    \begin{cases}
          \max\{\hspace{1mm}({\mathcal{T}}_{vv})_{i,j,k} (x)
    \mid x \in \mathbf{X}\} \hspace{2mm} \text{if} \hspace{2mm} j,k\in \mathcal{J}_{v}\\
        \max\{\hspace{1mm}({\mathcal{T}}_{vh})_{i,j,k} (x)
    \mid x \in \mathbf{X}\} \hspace{2mm} \text{if} \hspace{2mm} (j,k)\in \mathcal{J}_{v}\times \mathcal{J}_{h}\\                  
        \max\{\hspace{1mm}({\mathcal{T}}_{vc})_{i,j,k} (x)
    \mid x \in \mathbf{X}\} \hspace{2mm} \text{if} \hspace{2mm} (j,k)\in \mathcal{J}_{v}\times \mathcal{J}_{c}\\
    \max\{\hspace{1mm}({\mathcal{T}}_{hh})_{i,j,k} (x)
    \mid x \in \mathbf{X}\} \hspace{2mm} \text{if} \hspace{2mm} j,k\in \mathcal{J}_{h}\\
        \max\{\hspace{1mm}({\mathcal{T}}_{hc})_{i,j,k} (x)
    \mid x \in \mathbf{X}\} \hspace{2mm} \text{if} \hspace{2mm} (j,k)\in \mathcal{J}_{h}\times \mathcal{J}_{c}\\                  
    \max\{\hspace{1mm}({\mathcal{T}}_{cc})_{i,j,k} (x)
    \mid x \in \mathbf{X}\} \hspace{2mm} \text{if} \hspace{2mm} j,k\in \mathcal{J}_{c}
    \end{cases}
    \end{equation*}
with \(\mathcal{J}_v = \{1,\dots,d_0\},\quad \mathcal{J}_h = \{d_0+1,\dots,d_0+d\}\text{  and  }
\mathcal{J}_c = \{d_0+d+1,\dots,d_0+2d\}\) where the maps: \(x \mapsto ({\mathcal{T}}_{vv})_{i,j,k} (x)\), \(x \mapsto ({\mathcal{T}}_{vh})_{i,j,k} (x)\), and  \(x \mapsto ({\mathcal{T}}_{vc})_{i,j,k} (x)\) are given by:
\begin{itemize}
    \item $\begin{aligned}[t]
    ({\mathcal{T}}_{vv})_{i,j,k} (x)
    &= (\Theta^{(o)})_{i,j}(\Theta^{(o)})_{i,k}\,\sigma''(T_o^i)\tanh(c_t^i)
    + (\Theta^{(o)})_{i,j}\,\sigma'(T_o^i)\tanh'(c_t^i)\frac{\partial c_t^i}{\partial x^k} \\
    & + (\Theta^{(o)})_{i,k}\,\sigma'(T_o^i)\tanh'(c_t^i)\frac{\partial c_t^i}{\partial x^j}
    + \sigma(T_o^i)\tanh''(c_t^i)\frac{\partial c_t^i}{\partial x^j}\frac{\partial c_t^i}{\partial x^k} \\
    &+ \sigma(T_o^i)\tanh'(c_t^i)\frac{\partial^2 c_t^i}{\partial x^j \partial x^k}.
    \end{aligned}$\\
    \item $\begin{aligned}[t]
    ({\mathcal{T}}_{vh})_{i,j,k} (x)
    &= (\Theta^{(o)})_{i,j}\,\sigma'(T_o^i)\tanh'(c_t^i)\frac{\partial c_t^i}{\partial x^{k-d_0}} \\
    &+ (\Theta^{(o)})_{i,{k-d_0}}\,\sigma'(T_o^i)\tanh'(c_t^i)\frac{\partial c_t^i}{\partial x^j} \\
    &+ \sigma(T_o^i)\tanh''(c_t^i)\frac{\partial c_t^i}{\partial x^j}\frac{\partial c_t^i}{\partial x^{k-d_0}} + \sigma(T_o^i)\tanh'(c_t^i)\frac{\partial^2 c_t^i}{\partial x^j \partial x^{k-d_0}}.
    \end{aligned}$\\
    \item \(({\mathcal{T}}_{vc})_{i,j,k} (x) = \sigma(T_o^i)\tanh''(c_t^i)\dfrac{\partial c_t^i}{\partial x^j}\dfrac{\partial c_t^i}{\partial x^{k-d_0-d}} + \sigma(T_o^i)\tanh'(c_t^i)\dfrac{\partial^2 c_t^i}{\partial x^j \partial x^{k-d_0-d}}.\)
\end{itemize}

and the maps \(x \mapsto ({\mathcal{T}}_{hh})_{i,j,k} (x)\), and \(x \mapsto ({\mathcal{T}}_{hc})_{i,j,k} (x)\) are given by:
\begin{itemize}
    \item $\begin{aligned}[t]
    ({\mathcal{T}}_{hh})_{i,j,k} (x)
    &= (U^{(o)})_{i,j-d_0}(U^{(o)})_{i,k-d_0}\,\sigma''(T_o^i)\tanh(c_t^i) \\
    &+ (U^{(o)})_{i,j-d_0}\,\sigma'(T_o^i)\tanh'(c_t^i)\frac{\partial c_t^i}{\partial x^{k-d_0}} \\
    &+ (U^{(o)})_{i,k-d_0}\,\sigma'(T_o^i)\tanh'(c_t^i)\frac{\partial c_t^i}{\partial x^{j-d_0}}\\
    &+ \sigma(T_o^i)\tanh''(c_t^i)\frac{\partial c_t^i}{\partial x^{j-d_0}}\frac{\partial c_t^i}{\partial x^{k-d_0}} \\
    & + \sigma(T_o^i)\tanh'(c_t^i)\frac{\partial^2 c_t^i}{\partial x^{j-d_0} \partial x^{k-d_0}}.
    \end{aligned}$\\

    \item $\begin{aligned}[t]
    ({\mathcal{T}}_{hc})_{i,j,k}(x)
    &= \sigma(T_o^i)\tanh''(c_t^i)\frac{\partial c_t^i}{\partial x^{j-d_0}}\frac{\partial c_t^i}{\partial x^{k-d_0-d}}\\
      &+ \sigma(T_o^i)\tanh'(c_t^i)\frac{\partial^2 c_t^i}{\partial x^{j-d_0} \partial x^{k-d_0-d}}.
\end{aligned}$
\end{itemize}

the maps \(x \mapsto ({\mathcal{T}}_{cc})_{i,j,k} (x)\), and \(x \mapsto c^i_t (x)\) are given by:
\begin{itemize}
    \item $\begin{aligned}[t]
    ({\mathcal{T}}_{cc})_{i,j,k}(x)
    &= \sigma(T_o^i)\tanh''(c_t^i)\frac{\partial c_t^i}{\partial x^{j-d_0-d}}\frac{\partial c_t^i}{\partial x^{k-d_0-d}}\\
    &+ \sigma(T_o^i)\tanh'(c_t^i)\frac{\partial^2 c_t^i}{\partial x^{j-d_0-d} \partial x^{k-d_0-d}}.
\end{aligned}$\\
\item \(c^i_t = f^i_t \odot c^i_{t-1} + I_t^i \odot g^i_t\)
\end{itemize}

with:
\[
T_o^i = \sum_{p=1}^{d_0}(\Theta^{(o)})_{ip}.v_{w_t}^p + \sum_{q=1}^{d}(U^{(o)})_{iq}h_{t-1}^q + b_i^{(o)}.
\]
Here, \(\sigma'\) and \(\sigma''\) denote the first and second derivatives of the sigmoid function, and \(\tanh'\) and \(\tanh''\) denote those of the hyperbolic tangent. the parameters $\Theta^{(o)} \in \mathbb{R}^{d\times d_0}$ and $ U^{(o)} \in \mathbb{R}^{d\times d}$ are the input-hidden weights, hidden-hidden weights, and biases, respectively.

\end{proposition}

The quantity \((\mathcal{T})_{i,j,k}\) measures how the LSTM output coordinate \(i\) changes when two input components \(j\) and \(k\) vary together. 
The different cases $(vv, vh, vc,$ $hh, hc, cc)$ correspond to interactions between the current input, the previous hidden state, and the cell state. 
Each term reflects the combined effect of the output gate and the cell state, involving first and second order derivatives of the activation functions. 
The use of minimum and maximum over all inputs captures the worst-case behavior, showing that the S-GBT depends on both the gating mechanism and the nonlinear structure of the LSTM.\\

\textbf{Convolutional Neural Network (CNN)}. TextCNN \cite{Kim14f} employs a one-dimensional convolution followed by a max-pooling operation to extract informative features from text.\\ 
Consider a sentence composed of \(N\) word embeddings 
\(v_x = \langle v_{w_1}, v_{w_2}, \ldots, v_{w_N} \rangle\), 
where \(v_x\in\mathbf{V}\subseteq\mathbb{R}^{N d_0}\) and $\mathbf{V}$ denotes the domain of word embeddings, and each 
\(v_{w_t} \in \mathbb{R}^{d_0}\) is the embedding of the word at position \(t\).

The convolution and max-pooling operations are defined as follows. For \(k_i \in \mathcal{K}=\{k_1,\ldots,k_m\}\subset \mathbb{Z}_{\geq2}\) and \(t \in \{1,\ldots,N-k_i+1\}\),
\begin{equation*}\label{eq:conv}
    \mathcal{C}^{(k_i)}(v_x)_t = \phi \left( b^{(k_i)} + \sum_{l=0}^{k_i-1} W^{(k_i)}_{:,:,l}\, v_{w_{t+l}} \right),
\end{equation*}
\begin{equation}\label{eq:pool}
    \mathcal{P}^{(k_i)}(\mathcal{C}^{(k_i)}(v_x)) = \max_{t=1}^{N-k_i+1} \left( \mathcal{C}^{(k_i)}(v_x)_t \right),
\end{equation}
where \(k_i\) denotes the kernel size, \(\phi\) is the Tanh activation function, \(b^{(k_i)} \in \mathbb{R}^{d}\) is the bias vector, and \(W^{(k_i)} \in \mathbb{R}^{d \times d_0 \times k_i}\) represents the convolutional filter weights, 
and \(W^{(k_i)}_{:,:,l} \in \mathbb{R}^{d \times d_0}\) denotes the \(l\)-th slice of the tensor 
\(W^{(k_i)}\) along its third dimension. The term \(v_{w_{t+l}}\) corresponds to the word embedding at position \(t+l\).

Using Eq.~(\ref{eq:pool}), the CNN layer can be expressed as an input-output mapping \(\mathcal{F}\), similar to Eq.~(\ref{eq:orig}), defined by
\begin{equation}\label{eq:cnn}
    \begin{array}{l}
    \mathcal{F}: \quad \mathbf{V} \to \mathbb{R}^{|\mathcal{K}|\,d} \\
    \qquad v_x \mapsto \bigoplus_{k=k_1}^{k_m} \mathcal{P}^{(k)} \circ \mathcal{C}^{(k)} (v_x),
    \end{array}
\end{equation}
where \(\bigoplus\) denotes the concatenation operator and \(|\mathcal{K}|\) is the cardinality of the kernel set \(\mathcal{K}\).

The following result provides an explicit expression for the S-GBT of the CNN layer with a detailed proof provided in Appendix \ref{sec:proofs}. 
\begin{proposition}
\label{prop:cnn}
    Consider the  map \(\mathcal{F}\) defined in Eq.~\eqref{eq:cnn}. For any output index \(i \in \{1,\dots,|\mathcal{K}|\cdot d\}\) and any pair of input indices \(j,k \in \{1,\dots,Nd_0\}\), the S-GBT of the map \(\mathcal{F}\) is given by:
\begin{equation}
    \label{eq:sgbtcnn}
    \begin{aligned}
        (\mathcal{T})_{i,j,k} = &\left( \max_{t=1}^{N-k_p+1} \left\| W_{\beta(i,1,d),\,\beta(j,t,d_0),\,\alpha(j,t,d_0)}^{(k_p)} \right\| \right) \\
        &\cdot  \left( \max_{t=1}^{N-k_p+1} \left\| W_{\beta(i,1,d),\,\beta(k,t,d_0),\,\alpha(k,t,d_0)}^{(k_p)} \right\| \right) \cdot \sup_{z \in \mathbb{R}} |\phi''(z)|,
    \end{aligned}
\end{equation}
where \(k_p = k_{\alpha(i,1,d)}\) is the kernel size associated with output index \(i\), and the functions \(\alpha,\beta\) are defined as follow:
\[
\alpha(i,a,d) = \left\lfloor \frac{i-a}{d} \right\rfloor + 1, \qquad
\beta(i,a,d) = 1 + ((i-a) \bmod d).
\]
The term \(\sup_{z \in \mathbb{R}} |\phi''(z)|\) denotes the supremum of the absolute value 
of the second derivative of the activation function \(\phi\) over all real inputs. 
For the hyperbolic tangent activation, one has
\(
\sup_{z \in \mathbb{R}} |\tanh''(z)| = \frac{4}{3\sqrt{3}}.
\)
\end{proposition}

The expression of \((\mathcal{T})_{i,j,k}\) reflects how the second-order sensitivity of the CNN output 
with respect to two input coordinates \(j\) and \(k\) decomposes into three factors. 
The first two terms correspond to the maximal influence of the input coordinates \(j\) and \(k\) 
through the convolutional filters, captured by the largest weights across all valid positions \(t\). 
The maximization arises from the max-pooling operation, which selects the most responsive feature 
over the sequence. The last term, \(\sup_{z \in \mathbb{R}} |\phi''(z)|\), captures the nonlinearity of the activation function. Overall, the bound highlights that the S-GBT is governed by the strongest filter responses and the curvature of the activation, independently of the sentence length.

\section{Experiments}
This section evaluates the proposed S-GBT method through a series of experiments. We compare its performance with five defense baselines under three types of adversarial attacks on Bidirectional Long Short-Term Memory (BiLSTM) and CNN models.

\subsection{Experimental Setup}

\textbf{Datasets:} We conduct experiments on two benchmark datasets. \textit{IMDB} \cite{maas11learning} is a binary sentiment classification dataset containing 25,000 training and 25,000 testing movie reviews. \textit{Yahoo! Answers} \cite{zhang15character} is a large-scale topic classification dataset with 1,400,000 training samples and 50,000 test samples distributed over 10 classes.\\
\textbf{Defense Baselines:} We compare our method with standard training and five representative adversarial defense approaches. These include IBP \cite{jia19certified} and GBM \cite{bouri2025bridging}, which are certified defense methods; SEM \cite{wang21natural}, an input transformation-based approach; and ATFL \cite{wang21adversarial} and ASCC \cite{dong21towards}, two adversarial training methods based on GloVe embeddings.\\
\textbf{Attack Methods:} To assess robustness, we consider three strong adversarial attacks: GA (Genetic Attack) \cite{alzantot-etal-2018-generating}, PWWS (Probability Weighted Word Saliency) \cite{shuhuai19generating}, and PSO (Particle Swarm Optimization) \cite{yuan20pso}. PWWS selects synonyms from WordNet, GA relies on counter-fitted embeddings, and PSO uses sememe-based substitutions from HowNet. Due to the high computational cost of these attacks, adversarial examples are generated from 1000 randomly selected test samples for each dataset. All attacks are implemented using the OpenAttack framework \cite{zeng-etal-2021-openattack}.\\
\textbf{Perturbation Setting:} Following \cite{jia19certified} and \cite{dong21towards}, we set \(k=8\) as the number of nearest neighbors selected within a Euclidean distance \(d_e = 0.5\) in the GloVe embedding space.\\
\textbf{Model Setting:} We use pre-trained GloVe embeddings to represent words as 300-dimensional vectors, which serve as the input layer for both BiLSTM and CNN models.\\
\textbf{Evaluation Metrics:} We report two metrics. Clean Accuracy (CA) measures the classification accuracy on the full clean test set. Accuracy Under Attack (AUA) denotes the accuracy under a given adversarial attack, evaluated on random 1000 test samples that are correctly classified by the model before the attack. Formally, let \(\mathcal{D}_{\mathrm{sub}} \subset \mathcal{D}_{\mathrm{test}}\) be a subset of size 1000 such that 
\(f(x)=y\) for all \((x,y)\in \mathcal{D}_{\mathrm{sub}}\). Then, the AUA is defined as
\[
\mathrm{AUA} = \frac{1}{|\mathcal{D}_{\mathrm{sub}}|} 
\sum_{(x,y)\in \mathcal{D}_{\mathrm{sub}}} \mathbf{1}\bigl(f(x^{\mathrm{adv}})=y\bigr),
\]
where \(x^{\mathrm{adv}}\) denotes the adversarial example generated from \(x\), 
and \(\mathbf{1}(\cdot)\) is the indicator function.

\subsection{Main Results}

\textbf{Performance on CNN and BiLSTM:} 

Table~\ref{tab:defense} reports the performance of BiLSTM and CNN models on the \textit{IMDB} and \textit{Yahoo! Answers} datasets. Each row corresponds to a defense method, while each column represents an attack. A method is considered effective if it achieves high AUA while maintaining strong performance on clean data.

Our approach consistently outperforms existing defenses, with notable gains under the PWWS and PSO attacks. For the BiLSTM model, it improves AUA by up to \textbf{9.3\%} on \textit{IMDB} and \textbf{21.0\%} on \textit{Yahoo! Answers}. For the CNN model, it achieves improvements of up to \textbf{8.0\%} on \textit{IMDB} and \textbf{23.4\%} on \textit{Yahoo! Answers}. In addition, the proposed method maintains competitive clean accuracy on both datasets. These results demonstrate its ability to enhance robustness while preserving strong performance in both clean and adversarial settings.

\begin{table}[t]
    \centering
    \caption{Clean Accuracy (CA) (\%) and Accuracy Under Attack (AUA) (\%) for CNN and BiLSTM models on two datasets. The \textit{Clean} columns report CA on the full test set. For each attack, the best result is shown in \textbf{bold} and the second best in \underline{underline}. The last row in each block reports the accuracy improvement of our method over the strongest baseline.}
    \resizebox{\textwidth}{!}{
    \begin{tabular}{llcccccccc}
        \toprule
         \multirow{2}{*}{Dataset} & \multirow{2}{*}{Defense} & \multicolumn{4}{c}{CNN} & \multicolumn{4}{c}{BiLSTM}\\
         \cmidrule(lr){3-6} \cmidrule(lr){7-10}
          ~ & ~ & Clean & PWWS & GA & PSO & Clean & PWWS & GA & PSO \\
         \midrule
         \multirow{7}{*}{\textit{IMDB}} &
         Standard &  \underline{89.7} & ~~0.6 & ~~2.6 & ~~1.4 &  \underline{89.1} & ~~0.2 & ~~1.6 &  ~~0.3 \\
         &IBP \cite{jia19certified}& 81.7 &  {75.9} &  {76.0} &  {75.9} & 77.6 & 67.5 & 67.8 & 67.6 \\
         & ATFL \cite{wang21adversarial}& 85.0 & 63.6 & 66.8 & 64.7 & 85.1 & 72.2 & 75.5 & 74.0\\
         & SEM \cite{wang21natural}& 87.6 & 62.2 & 63.5 & 61.5 & 86.8 & 61.9 & 63.7 & 62.2\\
         & ASCC \cite{dong21towards} & 84.8 & 74.0 & 75.5 & 74.5 & 84.3 &  {74.2} &  {76.8} &  {75.5}\\
         & GBM \cite{bouri2025bridging} & \textbf{90.2} & \underline{76.3} & \underline{76.6} & \underline{84.3} & \textbf{89.6} & {76.8} & {77.8} & {84.3} \\
         \midrule
         & S-GBT & {79.4} & \textbf{80.3} & \textbf{83.0} & \textbf{92.3} & {88.4} & \textbf{86.1} & \textbf{85.6} & \textbf{92.4} \\
         & & & \textcolor{blue}{\(\uparrow\)5.7} & \textcolor{blue}{\(\uparrow\)6.4} & \textcolor{blue}{\(\uparrow\)8.0} & & \textcolor{blue}{\(\uparrow\)9.3} & \textcolor{blue}{\(\uparrow\)7.8} & \textcolor{blue}{\(\uparrow\)8.1} \\
         \midrule
         \multirow{7}{*}{\textit{\shortstack{Yahoo!\\ Answers}}} &
         Standard &  {72.6} & ~~6.8 & ~~7.2 & ~~4.9 & \textbf{74.7} & 12.2 & ~~9.6 & ~~6.5 \\
         & IBP \cite{jia19certified} & 63.1 & 54.9 & 54.9 & 54.8 & 54.3 & 47.3 & 47.6 & 47.0 \\
         & ATFL \cite{wang21adversarial} & 72.5 &  {62.5} &  {63.1} &  {62.5} &  \underline{73.6} &  {61.7} & 60.8 & 60.3\\ 
         & SEM \cite{wang21natural}& 70.1 & 53.8 & 52.4 & 51.9 & 72.3 & 57.0 & 56.1 & 55.4 \\
         & ASCC \cite{dong21towards} & 69.0 & 58.4 & 59.6 & 58.5 & 70.7 &  {61.7} &  {62.3} &  {61.9} \\
         & GBM \cite{bouri2025bridging} & {\textbf{72.8}} & \underline{66.2} & \underline{66.1} & \underline{67.3} & 73.0 & \underline{67.0} & \underline{67.0} & \underline{68.6} \\
         \midrule
         & S-GBT & 63.3 & \textbf{88.6} & \textbf{89.0} & \textbf{90.7} & 66.4 & \textbf{88.0} & \textbf{87.5} & \textbf{89.8} \\
         &  & & \textcolor{blue}{\(\uparrow\)15.8} & \textcolor{blue}{\(\uparrow\)22.9} & \textcolor{blue}{\(\uparrow\)23.4} & & \textcolor{blue}{\(\uparrow\)21.0} & \textcolor{blue}{\(\uparrow\)20.5} & \textcolor{blue}{\(\uparrow\)21.2} \\
         \bottomrule
    \end{tabular}
    }
    \label{tab:defense}
\end{table}
\subsection{Clean Accuracy versus Robust Accuracy}

\textbf{Hyper-parameter Study:} 
We analyze the impact of the hyperparameters \(\beta\) and \(\gamma\) in Eq.~\ref{eq:objective_sgbt}, 
which control the trade-off between the cross-entropy loss, GBM regularization, and S-GBT regularization. 
Figure~\ref{fig:hyperparam} reports the Accuracy Under Attack (AUA) against PWWS for the BiLSTM model on the IMDB dataset over different \((\beta,\gamma)\) values. 
A similar study is conducted for the CNN model.

Overall, S-GBT consistently achieves strong performance across a wide range of \((\beta,\gamma)\) values, 
outperforming baseline methods while remaining relatively stable to hyperparameter variations. 
This indicates that the proposed regularization effectively improves robustness without requiring fine-tuned parameter choices.

For BiLSTM, the best performance is obtained at \(\beta=0.1\) and \(\gamma=0.3\), achieving an AUA of \(82.0\%\). 
For CNN, fixing \(\beta=0.1\), the highest AUA (\(86.1\%\)) is also achieved at \(\gamma=0.3\). 
These results suggest that a moderate level of S-GBT regularization provides the best trade-off between accuracy and robustness.

\begin{figure}[t]
    \centering
    \includegraphics[width=12cm]{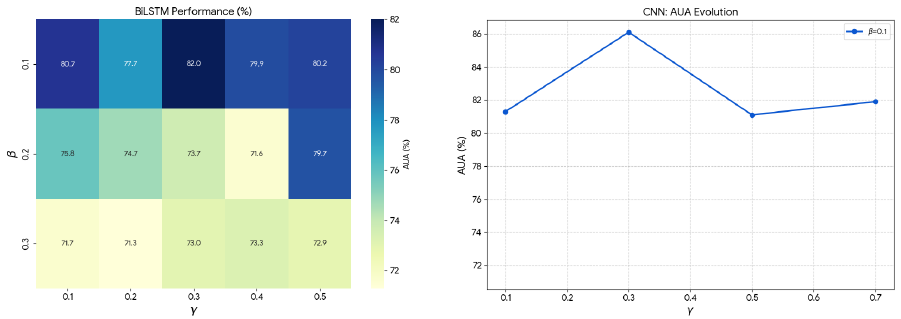}
    \caption{Hyperparameter study on the IMDB dataset for BiLSTM (left) and CNN (right) under the PWWS attack. The left heatmap shows the Accuracy Under Attack (AUA) of BiLSTM for different combinations of \(\beta\) and \(\gamma\), while the right plot reports the AUA of CNN as a function of \(\gamma\) with \(\beta=0.1\). Across a wide range of hyperparameter values, S-GBT consistently achieves high robustness and outperforms baseline methods, 
demonstrating stability with respect to \(\beta\) and \(\gamma\).}
    \label{fig:hyperparam}
\end{figure}

\subsection{Detailed Setup}
In the BiLSTM layer, the parameters $\Theta^{(\text{gate})}$ and $U^{(\text{gate})}$ are shared across all time steps. For classification, the final hidden state $h_N$ is used as the representation. Accordingly, the S-GBT is computed only for the last cell in both the forward and backward directions. The detailed S-GBT training procedure for the BiLSTM model is summarized in Table~\ref{tab:conf}.

\begin{table}[ht]
    \centering
    \caption{Training configuration and hyperparameters of the GBM training method for BiLSTM model.}
    \begin{tabular}{lc|c}
    \hline
    \textbf{Dataset} & IMDB & Yahoo! Answers \\
    \hline \hline
    \textbf{Optimizer} & \multicolumn{2}{c}{$\operatorname{Adam}$} \\
    \hline
    \textbf{Batch size} & \multicolumn{2}{c}{64} \\
    \hline
    \textbf{Hidden size} & \multicolumn{2}{c}{64} \\
    \hline
    \textbf{Learning rate} & \multicolumn{2}{c}{$10^{-3}$} \\
    \hline
    \textbf{Weight decay} & \multicolumn{2}{c}{$10^{-4}$} \\
    \hline
    \textbf{Max length} & 512 & 256 \\
    \hline
    \textbf{Early Stopping} & \multicolumn{2}{c}{Yes} \\
    \hline
    \end{tabular}
    \label{tab:conf}
\end{table}

The CNN model is trained using the S-GBT approach as describe the Table \ref{tab:conf_cnn}.
All experiments were conducted on a single NVIDIA A100 80GB GPU from Toubkal Supercomputer \cite{kissami2025toubkal}. 
\begin{table}[h]
    \centering
    \caption{Hyperparameter Settings and Training Configuration for GBM in CNN Model.}
    
    \begin{tabular}{lc|c}
    \hline
    \textbf{Dataset} & IMDB & Yahoo! Answers \\
    \hline \hline
    \textbf{Optimizer} & \multicolumn{2}{c}{$\operatorname{Adam}$} \\
    \hline
    \textbf{Batch size} & \multicolumn{2}{c}{64} \\
    \hline
    \textbf{Num kernels} & \multicolumn{2}{c}{128} \\
    \hline
    \textbf{Learning rate} & \multicolumn{2}{c}{$10^{-4}$} \\
    \hline
    \textbf{Weight decay} & \multicolumn{2}{c}{$10^{-4}$} \\
    \hline
    \textbf{Kernel sizes} & \multicolumn{2}{c}{(3, 4, 5)} \\
    \hline
    \textbf{Max length} & 512 & 256 \\
    \hline
    \textbf{Early Stopping} & \multicolumn{2}{c}{Yes} \\
    \hline
    \end{tabular}
    \label{tab:conf_cnn}
\end{table}

\section{Discussion}

The results in Table~\ref{tab:defense} indicate that incorporating S-GBT regularization consistently enhances adversarial robustness across different models and datasets. The improvements are particularly significant on the Yahoo dataset, where accuracy under attack increases by up to $23.4\%$ for CNN and more than $20\%$ for BiLSTM compared to GBM alone. On IMDB, the gains range from $5.7\%$ to $9.3\%$ depending on the attack. These results confirm that S-GBT effectively reduces model curvature, resulting in smoother decision boundaries and improved robustness.

The strong results on Yahoo, which involves more classes and longer input sequences, indicate that S-GBT regularization is particularly effective in complex settings. 
While GBM reduces local sensitivity by penalizing gradient norms, S-GBT further stabilizes the model by controlling variations in the gradients (the curvature), which is crucial for robustness against stronger attacks such as PSO. Unlike greedy attacks such as PWWS, which replace words one at a time, PSO performs a global search over the synonym substitution space and can identify combinations of word replacements that jointly exploit interactions between perturbations. These interactions are captured by second-order (Hessian) terms in the output variation. 
By explicitly bounding these terms (Proposition~\ref{prop:s_gbt_robustness}), S-GBT limits the model’s sensitivity to such combinatorial perturbations, leading to substantial improvements of up to 23.4\% for CNN and 21.2\% for BiLSTM on Yahoo under the PSO attack. This highlights the importance of controlling curvature for defending against attacks that leverage higher-order effects.

Overall, S-GBT extends GBM by explicitly limiting curvature, resulting in lower sensitivity and improved robustness to adversarial attacks. The method remains simple to integrate, and especially effective on challenging datasets.
\section{Conclusion}
This paper presents the Smoothness Growth Bound Tensor (S-GBT), which extends the Growth Bound Matrix (GBM) framework for certified robustness against word substitution attacks. By bounding the second-order derivatives of the model output with respect to the input, S-GBT enables tighter control over output variations under perturbations. We derive explicit formulations for both BiLSTM and CNN architectures and incorporate GBM and S-GBT regularization into the training objective.\\
Experiments on the IMDB and Yahoo datasets show that S-GBT consistently improves adversarial robustness compared to GBM-only and other baselines. On the more complex Yahoo dataset, the method achieves gains of up to $23.4\%$ for CNN and over $20\%$ for BiLSTM under attack. These results indicate that controlling both gradient and curvature leads to smoother decision boundaries and stronger robustness. The proposed approach remains lightweight and performs well on high-dimensional and challenging datasets. By accounting for second-order effects, this work provides a step toward more reliable NLP models under adversarial settings.

\section*{Limitations}
A limitation of the proposed S-GBT regularization lies in its dependence on the model architecture. In some cases, the Hessian may vanish due to linear dynamics, making S-GBT ineffective. Similarly, for ReLU-based CNNs with max-pooling, the Hessian is zero almost everywhere, which limits the benefit of S-GBT unless smooth activations such as $\tanh$ and average pooling are used. Therefore, the choice of architecture plays an important role when applying this method.\\
Despite this limitation, the approach remains well-suited for word substitution attacks with fixed-length perturbations. Future work will focus on extending the method to handle a wider range of word-level attacks, including variable-length perturbations such as word deletion, in order to improve its applicability to more fine-grained text classification tasks.

\bibliographystyle{splncs04}
\bibliography{custom}
\newpage
\appendix
\section{Appendix}
\subsection{Robustness Analysis with S-GBT}
\label{sec:rob}
\subsubsection{Proof of Proposition \ref{prop:s_gbt_robustness}}

\begin{proof}
Consider the map \( \mathcal{F}: \mathbf{X} \rightarrow \mathbf{Y}\), where \(\mathbf{X} \subseteq \mathbb{R}^{n_x}\) and \(\mathbf{Y} \subseteq \mathbb{R}^{n_y}\). Each component \( \mathcal{F}^i: \mathbf{X}\subseteq \mathbb{R}^{n_x} \rightarrow \mathbb{R}\) is a scalar function. Given an input \( x \in \mathbf{X} \), a perturbation vector \( {\delta} \in \mathbb{R}^{n_x} \) and consider the perturbed input $x'\in \mathbf{X}$ such that $x'=x+{\delta}$. 

By Taylor’s theorem with Lagrange remainder, there exists $\xi$ on the line segment between $x$ and $x'$ such that
\begin{equation*}
    \mathcal{F}^i(x') = \mathcal{F}^i(x) + \nabla\mathcal{F}^i(x)^\top\delta
+ \frac12\delta^\top\nabla^2\mathcal{F}^i(\xi)\delta.
\end{equation*}

Taking norms on both sides and applying the triangle inequality, one gets:
\begin{equation}
\label{eq:prop1prof}
    \left\lVert\mathcal{F}^i(x')-\mathcal{F}^i(x)\right\rVert
\le \left\lVert\nabla\mathcal{F}^i(x)^\top\delta\right\rVert
+ \frac12\left\lVert\delta^\top\nabla^2\mathcal{F}^i(\xi)\delta\right\rVert.
\end{equation}

Since \(\nabla\mathcal{F}^i(c) = \left( \dfrac{\partial \mathcal{F}^i}{\partial x^1}(c),\dots,\dfrac{\partial \mathcal{F}^i}{\partial x^{n_x}}(c)\right)^T\) and \(\nabla^2\mathcal{F}^i(\xi)\) is the $n_x\times n_x$ matrix evaluated at $\xi$ whose $(j,k)$-th entry is $\dfrac{\partial^2\mathcal{F}^i}{\partial x^j\partial x^k}(\xi)$. By the Definition \ref{def:gbm} of GBM and the Definition \ref{def:s_gbt} of S-GBT, we have for $\xi$ on the line segment between $x$ and $x'$, and for all $i\in\{1,\dots,n_y\}$, $j,k\in\{1,\dots,n_x\}$:

\[
\left\lVert \dfrac{\partial \mathcal{F}^i}{\partial x^j}(\xi)\right\rVert \leq (\mathcal{M})_{ij} \quad \text{and} \quad \left\lVert\frac{\partial^2\mathcal{F}^i}{\partial x^j\partial x^k}(x)\right\rVert \le (\mathcal{T})_{i,j,k}
\]

Hence, one gets from Inequality \eqref{eq:prop1prof} that:
\[
\begin{aligned}
    \left\lVert\mathcal{F}^i(x')-\mathcal{F}^i(x)\right\rVert
&\le \left\lVert\nabla\mathcal{F}^i(x)^\top\delta\right\rVert
+ \frac12\left\lVert\delta^\top\nabla^2\mathcal{F}^i(\xi)\delta\right\rVert.\\
&\le \sum\limits_{j=1}^{n_x}(\mathcal{M})_{i,j}\|{\delta_j}\| + \frac12 \sum_{j,k}\left\lVert\frac{\partial^2\mathcal{F}^i}{\partial x^j\partial x^k}(\xi)\right\rVert\|\delta_j\|\|\delta_k\|\\
& \le \sum\limits_{j=1}^{n_x}(\mathcal{M})_{i,j}\|{\delta_j}\| + \frac12 \sum_{j,k}(\mathcal{T})_{i,jk}\|\delta_j\|\|\delta_k\|
\end{aligned}
\]
Then, we conclude the following result:
\[
\begin{aligned}
\mathcal{F}^i(x)
- \sum_{j=1}^{n_x} (\mathcal{M})_{ij} \| \delta_j \|
-& \frac{1}{2} \sum_{j,k=1}^{n_x} (\mathcal{T})_{ijk} \| \delta_j \| \| \delta_k \|
\le \mathcal{F}^i(x') \\
&\le \mathcal{F}^i(x)
+ \sum_{j=1}^{n_x} (\mathcal{M})_{ij} \| \delta_j \|
+ \frac{1}{2} \sum_{j,k=1}^{n_x} (\mathcal{T})_{ijk} \| \delta_j \| \| \delta_k \| .
\end{aligned}
\]
which concludes the proof.\hfill $\square$
\end{proof}

\subsection{Auxiliary Results}
\begin{proposition}
\label{prop:interval}
    Consider the mapping
\begin{align*}
    \varphi: \mathcal{X} &\to \mathbb{R}, \\ 
    x &\mapsto a^\top x,
\end{align*}
where $\mathcal{X} = \prod_{r=1}^d [\underline{x}_r,\overline{x}_r] \subset \mathbb{R}^d$ 
and $a \in \mathbb{R}^d$ is a fixed vector. For any $x \in \mathcal{X}$, the value of $\varphi(x)$ is bounded within the interval
\[
\varphi(x) \in [\,\underline{a^\top x}, \overline{a^\top x}\,],
\]
where the lower and upper bounds are given by
\[
\underline{a^\top x} = \sum_{r=1}^d a_r \cdot
\begin{cases}
\underline{x}_r & \text{if } a_r \geq 0,\\
\overline{x}_r & \text{if } a_r < 0,
\end{cases}
\quad
\overline{a^\top x} = \sum_{r=1}^d a_r \cdot
\begin{cases}
\overline{x}_r & \text{if } a_r \geq 0,\\
\underline{x}_r & \text{if } a_r < 0.
\end{cases}
\]
\end{proposition}

\begin{proposition}
\label{prop:T}

    For each \(i \in \{1,\ldots,d\}\), consider the mapping \((v_{w_t}, h_{t-1}) \mapsto T^i_{gate}(v_{w_t}, h_{t-1})\) defined by
\begin{equation*}
    T^i_{gate} = \sum_{p=1}^{d_0} (\Theta^{(gate)})_{pi}\, v_{w_t}^{p} + \sum_{q=1}^{d} (U^{(gate)})_{qi}\, h_{t-1}^{q} + b^{(gate)}_i,
\end{equation*}
where \(\Theta^{(gate)} \in \mathbb{R}^{d \times d_0}\), \(U^{(gate)} \in \mathbb{R}^{d \times d}\), and \(b^{(gate)} \in \mathbb{R}^{d}\), for \(gate \in \{f, I, g, o\}\).

Assume that \(v_{w_t} \in [\underline{v}, \overline{v}] \subseteq \mathbb{R}^{d_0}\) and \(h_{t-1} \in [\underline{h}, \overline{h}] \subseteq \mathbb{R}^{d}\). Then \(T^i_{gate}\) is bounded as
\[
T^i_{gate} \in [\underline{T}^i, \overline{T}^i],
\]
where the bounds are given by
\begin{align*}
    \underline{T}^i &= \sum_{p=1}^{d_0} \underline{\alpha}_{pi} + \sum_{q=1}^{d} \underline{\beta}_{qi} + b^{(gate)}_i, \\
    \overline{T}^i &= \sum_{p=1}^{d_0} \overline{\alpha}_{pi} + \sum_{q=1}^{d} \overline{\beta}_{qi} + b^{(gate)}_i,
\end{align*}
with
\begin{align*}
    \underline{\alpha}_{pi} &= 
    \begin{cases}
    (\Theta^{(gate)})_{pi}\, \underline{v}_{w_t}^{p}, & \text{if } (\Theta^{(gate)})_{pi} \geq 0, \\
    (\Theta^{(gate)})_{pi}\, \overline{v}_{w_t}^{p}, & \text{otherwise},
    \end{cases} \\
    \overline{\alpha}_{pi} &= 
    \begin{cases}
    (\Theta^{(gate)})_{pi}\, \overline{v}_{w_t}^{p}, & \text{if } (\Theta^{(gate)})_{pi} \geq 0, \\
    (\Theta^{(gate)})_{pi}\, \underline{v}_{w_t}^{p}, & \text{otherwise},
    \end{cases} \\
    \underline{\beta}_{qi} &= 
    \begin{cases}
    (U^{(gate)})_{qi}\, \underline{h}_{t-1}^{q}, & \text{if } (U^{(gate)})_{qi} \geq 0, \\
    (U^{(gate)})_{qi}\, \overline{h}_{t-1}^{q}, & \text{otherwise},
    \end{cases} \\
    \overline{\beta}_{qi} &= 
    \begin{cases}
    (U^{(gate)})_{qi}\, \overline{h}_{t-1}^{q}, & \text{if } (U^{(gate)})_{qi} \geq 0, \\
    (U^{(gate)})_{qi}\, \underline{h}_{t-1}^{q}, & \text{otherwise}.
    \end{cases}
\end{align*}
\end{proposition}

\subsection{Proofs of Architecture-Specific S-GBT}
\label{sec:proofs}
In this section, we provide the proofs of the main results, specifically Proposition~2 and Proposition~3. For CNN architectures, the S-GBT can be directly computed from the learned parameters, making the derivation straightforward. In contrast, for LSTM models, the computation is more complex and requires bounding several intermediate functions.
\subsubsection{Proof of Proposition \ref{prop:lstm}}
\begin{proof}
Consider the map $\mathcal{F}$ describing the input-output model on an LSTM cell defined in (\ref{eq:lstm}), where the input is \(\ x=(v_{w_t}, h_{t-1}, c_{t-1}) \in \mathbf{V} \times \mathbf{H} \times \mathbf{C} \subseteq \mathbb{R}^{d_0+2d}\) and the output is \(\,y_t = \mathcal{F}(v_{w_t}, h_{t-1}, c_{t-1}) = o_t \odot \tanh(f_t \odot c_{t-1} + I_t \odot g_t) \in \mathbb{R}^d\), with,
\begin{align}
    I_t &= \sigma\left(\Theta^{(I)}. v_{w_t} + U^{(I)} .h_{t-1} + b^{(I)}\right)  \label{eq:it}\\
    f_t &= \sigma\left(\Theta^{(f)} .v_{w_t} + U^{(f)} .h_{t-1} + b^{(f)}\right)  \label{eq:ft}\\
    g_t &= \tanh\left(\Theta^{(g)} .v_{w_t} + .U^{(g)} h_{t-1} + b^{(g)}\right)  \label{eq:gt}\\
    o_t &= \sigma\left(\Theta^{(o)} .v_{w_t} + U^{(o)} .h_{t-1} + b^{(o)}\right)
     \label{eq:ot}
\end{align}where $\sigma$ denotes the sigmoid function and $\tanh$ denotes the hyperbolic tangent function. For $gate \in \{I,f,g,o\}$, the parameters $\Theta^{(gate)} \in \mathbb{R}^{d\times d_0} , U^{(gate)} \in \mathbb{R}^{d\times d}$ and $b^{(gate)} \in \mathbb{R}^{d}$ are the input-hidden weights, hidden-hidden weights, and biases, respectively.

Lets consider the following:
\begin{small}
    \begin{equation*}
    T_{gate}^i = \sum_{p=1}^{d_0} (\Theta^{(gate)})_{ip}.v_{w_t}^{p} + \sum_{q=1}^{d} (U^{(gate)})_{iq}.h_{t-1}^{q} + b_i^{(gate)}
\end{equation*}
\end{small}
with \(\sigma'\) and \(\tanh'\) are the derivative of the sigmoid and hyperbolic tangent functions, respectively.

Here, \(\mathbf{V}\) denotes the domain of the word vectors, \(\mathbf{H}\) denotes the domain of hidden states, and \(\mathbf{C}\) denotes the domain of cell states. The i-th component of the map $\mathcal{F}$ is given for $i \in \{1,\ldots,d\}$ by: 
$$\mathcal{F}^i(x)=o_t^i \odot \tanh(f^i_t \odot c^i_{t-1} + I^i_t \odot g^i_t). $$ 

We consider a fixed output component \(i\in\{1,\dots,d\}\) and two input indices \(j,k\in\{1,\dots,d_0+2d\}\). The input vector \(x\in \mathbf{V} \times \mathbf{H} \times \mathbf{C} \subseteq \mathbb{R}^{d_0+2d}\) is the concatenation of the word vector \(v_{w_t}\in\mathbb{R}^{d_0}\), the previous hidden state \(h_{t-1}\in\mathbb{R}^{d}\), and the previous cell state \(c_{t-1}\in\mathbb{R}^{d}\). We partition the index set into
\[
\mathcal{J}_v = \{1,\dots,d_0\},\quad
\mathcal{J}_h = \{d_0+1,\dots,d_0+d\},\quad
\mathcal{J}_c = \{d_0+d+1,\dots,d_0+2d\}.
\]

Based on the GBM expression, we have the following:
\[
\frac{\partial \mathcal{F}^i}{\partial x^j} = \frac{\partial o_t^i}{\partial x^j}\tanh(c_t^i) + o_t^i\,\tanh'(c_t^i)\frac{\partial c_t^i}{\partial x^j}.
\]

Then,
\begin{equation}
\label{eq:full}
    \begin{aligned}
    \frac{\partial^2 \mathcal{F}^i}{\partial x^j\partial x^k} &= 
\frac{\partial^2 o_t^i}{\partial x^j\partial x^k}\tanh(c_t^i)
+ \frac{\partial o_t^i}{\partial x^j}\tanh'(c_t^i)\frac{\partial c_t^i}{\partial x^k}
+ \frac{\partial o_t^i}{\partial x^k}\tanh'(c_t^i)\frac{\partial c_t^i}{\partial x^j}\\
&+ o_t^i\,\tanh''(c_t^i)\frac{\partial c_t^i}{\partial x^j}\frac{\partial c_t^i}{\partial x^k}
+ o_t^i\,\tanh'(c_t^i)\frac{\partial^2 c_t^i}{\partial x^j\partial x^k}.
\end{aligned}
\end{equation}

Now for the various cases of \(j\) and \(k\), we have the following:

Based on Eq.\eqref{eq:ot}, \(o_t^i = \sigma(T_o^i)\). Hence
\begin{equation}
    \label{eq:ott}
    \frac{\partial o_t^i}{\partial x^j} = W^{(o)}(i,j)\,\sigma'(T_o^i) \quad \text{and} \quad
\frac{\partial^2 o_t^i}{\partial x^j\partial x^k} = W^{(o)}(i,j)\,W^{(o)}(i,k)\,\sigma''(T_o^i),
\end{equation}
where
\[
W^{(o)}(i,j) = \begin{cases}
(\Theta^{(o)})_{ij} & \text{if } j\in\mathcal{J}_v,\\
(U^{(o)})_{i,j-d_0} & \text{if } j\in\mathcal{J}_h,\\
0 & \text{if } j\in\mathcal{J}_c.
\end{cases}
\]

Or, \(c_t^i = f_t^i c_{t-1}^i + I_t^i g_t^i\). So, applying the GBM expression yields the following: 
\begin{equation}
    \label{eq:devct}
    \frac{\partial c_t^i}{\partial x^j} = \frac{\partial f_t^i}{\partial x^j}c_{t-1}^i + f_t^i\frac{\partial c_{t-1}^i}{\partial x^j} + \frac{\partial I_t^i}{\partial x^j}g_t^i + I_t^i\frac{\partial g_t^i}{\partial x^j}.
\end{equation}
where,
\[
\frac{\partial f_t^i}{\partial x^j} = W^{(f)}(i,j)\,\sigma'(T_f^i),\quad
\frac{\partial I_t^i}{\partial x^j} = W^{(I)}(i,j)\,\sigma'(T_I^i),\quad
\frac{\partial g_t^i}{\partial x^j} = W^{(g)}(i,j)\,\tanh'(T_g^i).
\]

Based on Eq.\eqref{eq:devct}, we derive the following:
\begin{equation}
    \label{eq:ctt}
    \begin{aligned}
\frac{\partial^2 c_t^i}{\partial x^j \partial x^k}
&= W^{(f)}{(i,j)}W^{(f)}{(i,k)}\,\sigma''(T_f^i)\,c_{t-1}^i \\
&\quad + W^{(I)}{(i,j)}W^{(I)}{(i,k)}\,\sigma''(T_I^i)\,g_t^i \\
&\quad + W^{(I)}{(i,j)}W^{(g)}{(i,k)}\,\sigma'(T_I^i)\tanh'(T_g^i) \\
&\quad + W^{(I)}{(i,k)}W^{(g)}{(i,j)}\,\sigma'(T_I^i)\tanh'(T_g^i) \\
&\quad + W^{(g)}{(i,j)}W^{(g)}{(i,k)}\,I_t^i\,\tanh''(T_g^i).
\end{aligned}
\end{equation}

Based on Eqs.\eqref{eq:full}\eqref{eq:ott}\eqref{eq:ctt}, we conclude the following:\\
\textbf{Case 1:} if $j,k\in \mathcal{J}_{v}$
\[
    \begin{aligned}
    ({\mathcal{T}}_{vv})_{i,j,k} (x)
    &= (\Theta^{(o)})_{i,j}(\Theta^{(o)})_{i,k}\,\sigma''(T_o^i)\tanh(c_t^i)
    + (\Theta^{(o)})_{i,j}\,\sigma'(T_o^i)\tanh'(c_t^i)\frac{\partial c_t^i}{\partial x^k} \\
    & + (\Theta^{(o)})_{i,k}\,\sigma'(T_o^i)\tanh'(c_t^i)\frac{\partial c_t^i}{\partial x^j}
    + \sigma(T_o^i)\tanh''(c_t^i)\frac{\partial c_t^i}{\partial x^j}\frac{\partial c_t^i}{\partial x^k} \\
    &+ \sigma(T_o^i)\tanh'(c_t^i)\frac{\partial^2 c_t^i}{\partial x^j \partial x^k}.
    \end{aligned}
\]
Hence, it follows that:
\[
 \left\lVert\frac{\partial^2\mathcal{F}^i}{\partial x^j\partial x^k}(x)\right\rVert\le \max\bigl( |(\underline{{\mathcal{T}}_{vv}})_{i,j,k}|, |(\overline{{\mathcal{T}}_{vv}})_{i,j,k}| \bigr).
\]
\textbf{Case 2:} if $j,k\in \mathcal{J}_{h}$
\[
    \begin{aligned}
    ({\mathcal{T}}_{hh})_{i,j,k} (x)
    &= (U^{(o)})_{i,j-d_0}(U^{(o)})_{i,k-d_0}\,\sigma''(T_o^i)\tanh(c_t^i) \\
    &+ (U^{(o)})_{i,j-d_0}\,\sigma'(T_o^i)\tanh'(c_t^i)\frac{\partial c_t^i}{\partial x^{k-d_0}} \\
    &+ (U^{(o)})_{i,k-d_0}\,\sigma'(T_o^i)\tanh'(c_t^i)\frac{\partial c_t^i}{\partial x^{j-d_0}} \\
    &+\sigma(T_o^i)\tanh''(c_t^i)\frac{\partial c_t^i}{\partial x^{j-d_0}}\frac{\partial c_t^i}{\partial x^{k-d_0}} \\
    & + \sigma(T_o^i)\tanh'(c_t^i)\frac{\partial^2 c_t^i}{\partial x^{j-d_0} \partial x^{k-d_0}}.
    \end{aligned}
\]
Hence, it follows that:
\[
 \left\lVert\frac{\partial^2\mathcal{F}^i}{\partial x^j\partial x^k}(x)\right\rVert\le \max\bigl( |(\underline{{\mathcal{T}}_{hh}})_{i,j,k}|, |(\overline{{\mathcal{T}}_{hh}})_{i,j,k}| \bigr).
\]
\textbf{Case 3:} if $j,k\in \mathcal{J}_{v}\times \mathcal{J}_{h}$
\[
    \begin{aligned}
    ({\mathcal{T}}_{vh})_{i,j,k} (x)
    &= (\Theta^{(o)})_{i,j}\,\sigma'(T_o^i)\tanh'(c_t^i)\frac{\partial c_t^i}{\partial x^{k-d_0}} \\
    &+ (\Theta^{(o)})_{i,{k-d_0}}\,\sigma'(T_o^i)\tanh'(c_t^i)\frac{\partial c_t^i}{\partial x^j} \\
    &+ \sigma(T_o^i)\tanh''(c_t^i)\frac{\partial c_t^i}{\partial x^j}\frac{\partial c_t^i}{\partial x^{k-d_0}} + \sigma(T_o^i)\tanh'(c_t^i)\frac{\partial^2 c_t^i}{\partial x^j \partial x^{k-d_0}}.
    \end{aligned}
\]
Hence, it follows that:
\[
 \left\lVert\frac{\partial^2\mathcal{F}^i}{\partial x^j\partial x^k}(x)\right\rVert\le \max\bigl( |(\underline{{\mathcal{T}}_{vh}})_{i,j,k}|, |(\overline{{\mathcal{T}}_{vh}})_{i,j,k}| \bigr).
\]
\textbf{Case 4:} if $j,k\in \mathcal{J}_{v}\times \mathcal{J}_{c}$
\[
({\mathcal{T}}_{vc})_{i,j,k} (x) = \sigma(T_o^i)\tanh''(c_t^i)\frac{\partial c_t^i}{\partial x^j}\frac{\partial c_t^i}{\partial x^{k-d_0-d}} + \sigma(T_o^i)\tanh'(c_t^i)\frac{\partial^2 c_t^i}{\partial x^j \partial x^{k-d_0-d}}.
\]
Hence, it follows that:
\[
 \left\lVert\frac{\partial^2\mathcal{F}^i}{\partial x^j\partial x^k}(x)\right\rVert\le \max\bigl( |(\underline{{\mathcal{T}}_{vc}})_{i,j,k}|, |(\overline{{\mathcal{T}}_{vc}})_{i,j,k}| \bigr).
\]
\textbf{Case 5:} if $j,k\in \mathcal{J}_{h}\times \mathcal{J}_{c}$
\[
\begin{aligned}
    ({\mathcal{T}}_{hc})_{i,j,k}(x)
    &= \sigma(T_o^i)\tanh''(c_t^i)\frac{\partial c_t^i}{\partial x^{j-d_0}}\frac{\partial c_t^i}{\partial x^{k-d_0-d}}\\
    &+ \sigma(T_o^i)\tanh'(c_t^i)\frac{\partial^2 c_t^i}{\partial x^{j-d_0} \partial x^{k-d_0-d}}.
\end{aligned}
\]
Hence, it follows that:
\[
 \left\lVert\frac{\partial^2\mathcal{F}^i}{\partial x^j\partial x^k}(x)\right\rVert\le \max\bigl( |(\underline{{\mathcal{T}}_{hc}})_{i,j,k}|, |(\overline{{\mathcal{T}}_{hc}})_{i,j,k}| \bigr).
\]
\textbf{Case 6:} if $j,k\in \mathcal{J}_{c}$
\[
\begin{aligned}
    ({\mathcal{T}}_{cc})_{i,j,k}(x)
    &= \sigma(T_o^i)\tanh''(c_t^i)\frac{\partial c_t^i}{\partial x^{j-d_0-d}}\frac{\partial c_t^i}{\partial x^{k-d_0-d}}\\
    &+ \sigma(T_o^i)\tanh'(c_t^i)\frac{\partial^2 c_t^i}{\partial x^{j-d_0-d} \partial x^{k-d_0-d}}.
\end{aligned}
\]
Hence, it follows that:
\[
 \left\lVert\frac{\partial^2\mathcal{F}^i}{\partial x^j\partial x^k}(x)\right\rVert\le \max\bigl( |(\underline{{\mathcal{T}}_{cc}})_{i,j,k}|, |(\overline{{\mathcal{T}}_{cc}})_{i,j,k}| \bigr),
\]
which concludes the proof.\hfill $\square$
\end{proof}

The computation of the Smooth Growth Bound Tensor (S-GBT) for LSTM models relies on a sequence of analytical bounds combined with interval-based calculations. For each output component \(i\) and each pair of input indices \((j,k)\), the procedure starts by bounding the gate pre-activations \(T_o^i\), \(T_f^i\), \(T_I^i\), and \(T_g^i\) using Proposition~\ref{prop:T}.

Next, the corresponding activation functions and their derivatives are bounded over these intervals. In particular, we use known ranges for the sigmoid and hyperbolic tangent functions and their derivatives, namely: \(\sigma \in (0,1)\), \(\sigma' \in (0,\tfrac{1}{4}]\), \(\sigma'' \in [-\tfrac{1}{6}, \tfrac{1}{6}]\), \(\tanh \in (-1,1)\), \(\tanh' \in (0,1]\), and \(\tanh'' \in [-\tfrac{4}{3}, \tfrac{4}{3}]\). These bounds are applied to obtain interval estimates for each activation term and its derivatives.

Using the LSTM update equations, we then bound the cell state \(c_t^i\) and its first partial derivatives \(\frac{\partial c_t^i}{\partial x_j}\) through interval arithmetic using Proposition \ref{prop:interval}. These results are used to derive bounds for the second partial derivatives \(\frac{\partial^2 c_t^i}{\partial x_j \partial x_k}\), again using interval operations applied to the gate expressions.

Finally, all intermediate bounds are combined to obtain an interval enclosing \(\frac{\partial^2 F_i}{\partial x_j \partial x_k}\). The corresponding S-GBT entry is defined as the maximum absolute value within this interval. This process is repeated for all output indices \(i\) and input pairs \((j,k)\), with appropriate handling depending on whether the indices correspond to word embeddings, hidden states, or cell states. The resulting tensor provides certified element-wise bounds on the curvature of the LSTM model.

\subsubsection{Proof of Proposition \ref{prop:cnn}}
\begin{proof}
Consider the map \(\mathcal{F}\) defined in Eq.~\eqref{eq:cnn}, which characterizes the input-output transformation of a convolutional neural network (CNN). The input is a flattened embedding vector \(x = v_x \in \mathbb{R}^{N d_0}\), and the output is obtained by applying convolution and pooling operations across multiple kernel sizes. Specifically, the output is given by:
\[
y = \bigoplus_{k \in \mathcal{K}} \mathcal{P}^{(k)} \circ \mathcal{C}^{(k)}(v_x) \in \mathbb{R}^{|\mathcal{K}| \cdot d},
\]
where \(\mathcal{K} = \{k_1, \dots, k_m\}\) denotes the set of kernel sizes, and \(d\) is the number of output channels (filters) per kernel size.

For a given kernel size \(k_p \in \mathcal{K}\), the output of the corresponding convolution and max-pooling operation is a \(d\)-dimensional vector. The \(i\)-th component of this vector is given by:
\[
\mathcal{F}^i_p(x) = \max_{t = 1}^{N - k_p + 1} \phi\left(b_i^{(k_p)} + \sum_{l = 0}^{k_p - 1} W_{i,:,l}^{(k_p)} \cdot v_{w_{t + l}}\right),
\]
where \(W^{(k_p)} \in \mathbb{R}^{d \times d_0 \times k_p}\) is the convolutional kernel, \(b_i^{(k_p)}\) is the bias, and \(\phi\) is the activation function. After processing all kernel sizes, the final output vector is the concatenation of these vectors. Thus, each output index \(i \in \{1, \dots, |\mathcal{K}| \cdot d\}\) corresponds to a specific kernel \(k_{\alpha(i,1,d)}\) and a filter \(\beta(i,1,d)\) within that kernel, where \(\alpha\) and \(\beta\) are defined as:
\[
\alpha(i,a,d) = \left\lfloor \frac{i-a}{d} \right\rfloor + 1, \qquad
\beta(i,a,d) = 1 + ((i-a) \bmod d).
\]

Now consider two input indices \(j,k \in \{1,\dots,Nd_0\}\). Each input index corresponds to a word position and embedding dimension via the same functions: let \(p = \alpha(j,1,d_0)\), \(f = \beta(j,1,d_0)\) and similarly for \(k\). The second partial derivative \(\dfrac{\partial^2 \mathcal{F}^i}{\partial x^j \partial x^k}(x)\) is non-zero only if both \(j\) and \(k\) affect the same output through the convolution. In a neighborhood where the argmax of the max-pooling is constant (say at position \(t^*\)), the mapping simplifies to \(\mathcal{F}^i(x) = \phi(z_{i,t^*})\) with
\[
z_{i,t^*} = b_i^{(k_p)} + \sum_{l=0}^{k_p-1} W_{\beta(i,1,d),\,:,\,l}^{(k_p)} \cdot v_{w_{t^*+l}}.
\]
Since \(z_{i,t^*}\) is affine linear in the input, its second derivative vanishes. Applying the chain rule twice gives
\[
\frac{\partial^2 \mathcal{F}^i}{\partial x^j \partial x^k}(x) = \phi''(z_{i,t^*}) \cdot \frac{\partial z_{i,t^*}}{\partial x^j} \cdot \frac{\partial z_{i,t^*}}{\partial x^k}.
\]

The derivatives \(\dfrac{\partial z_{i,t^*}}{\partial x^j}\) are non-zero only if the word position \(p\) lies within the kernel window at time \(t^*\), i.e., \(t^* \le p \le t^*+k_p-1\). In that case,
\[
\left\lVert\frac{\partial z_{i,t^*}}{\partial x^j}\right\lVert = \left\lVert W_{\beta(i,1,d),\,f,\,t^*-p}^{(k_p)} \right\rVert.
\]
Over the entire input domain, the actual maximizing position \(t^*\) may vary. To obtain a bound valid for all inputs, we take all possible temporal positions that could be selected by the max-pooling. Hence,
\[
\left\lVert\frac{\partial z_{i,t^*}}{\partial x^j}\right\lVert \le \max_{t=1}^{N-k_p+1} \left\| W_{\beta(i,1,d),\,\beta(j,t,d_0),\,\alpha(j,t,d_0)}^{(k_p)} \right\|,
\]
and similarly for \(\dfrac{\partial z_{i,t^*}}{\partial x^k}\). \\
Combining these estimates and using \(\lVert\phi''(z)\lVert \le \sup_z \lVert\phi''(z)\lVert\), we obtain:
\[
\left\lVert\frac{\partial^2 \mathcal{F}^i}{\partial x^j \partial x^k}(x)\right\lVert
\le \left( \max_{t} \|W_{i,j,t}^{(k_p)}\| \right) \cdot \left( \max_{t} \|W_{i,k,t}^{(k_p)}\| \right) \cdot \sup_z \lVert\phi''(z)\lVert.
\]

Taking the supremum over all \(x \in \mathbf{X}\) yields the entry of the Smooth Growth Bound Tensor as stated in Eq.~\eqref{eq:sgbtcnn}, which concludes the proof.\hfill $\square$
\end{proof}
\end{document}